\documentclass{article}

\usepackage[nonatbib, preprint]{neurips_2023}
\usepackage{multirow} 

\usepackage{placeins}
\usepackage{wrapfig}
\usepackage{adjustbox}
\usepackage{tabularx}
\usepackage{times}
\usepackage{soul}
\usepackage{xcolor}
\usepackage{url}
\usepackage[hidelinks]{hyperref}
\usepackage[utf8]{inputenc}
\usepackage[small]{caption}
\usepackage{graphicx}
\usepackage{amsmath}
\usepackage{amsthm}
\usepackage{booktabs}

\newtheorem{definition}{Definition}

\usepackage{algorithm}
\usepackage[noend]{algpseudocode}
\usepackage{appendix}
\usepackage{etoolbox}
\usepackage{amsmath,amssymb,amsfonts}
\usepackage[normalem]{ulem}
\newcommand{\name}{{\textsc{GPatcher}}}
\usepackage{tabularx}
\usepackage{multirow}
\usepackage{linegoal}
\usepackage{thmtools, thm-restate}

\newif\ifshowcomment
\showcommenttrue
\newcommand{\zhou}[1]{\ifshowcomment 
{\textsf{\textcolor{red}{\textsuperscript{\textit{Zhou}}[#1]}}} 
\else \fi}
\NewDocumentCommand{\shuaicheng}{ mO{} }{\textcolor{blue}{\textsuperscript{\textit{Shuaicheng}}\textsf{\textbf{\small[#1]}}}}

\newcommand{\hidesout}[1]{}

\newcommand{\haohui}[1]{{\small\color{green}{\bf haohui: #1}}}
\newcommand{\si}[1]{{\small\color{brown}{\bf si: #1}}}
\newcommand{\zhu}[1]{{\small\color{magenta}{\bf zhu: #1}}}
\newcommand{\mkclean}{
  \renewcommand{\zhou}[1]{}
  \renewcommand{\zhu}[1]{}
  \renewcommand{\si}[1]{}
  \renewcommand{\shuaicheng}[1]{}
  \renewcommand{\haohui}[1]{}
}
\mkclean

\usepackage[utf8]{inputenc} 
\usepackage[T1]{fontenc}    
\usepackage{hyperref}       
\usepackage{url}            
\usepackage{booktabs}       
\usepackage{amsfonts}       
\usepackage{nicefrac}       
\usepackage{microtype}      
\usepackage{xcolor}         
\usepackage{tabularray}
\algnewcommand{\LeftComment}[1]{\Statex \(\triangleright\) #1}

\title{\name: A Simple and Adaptive MLP Model for Alleviating Graph Heterophily}

\author{%
  Shuaicheng Zhang*\\
  Department of Computer Science\\
  Virginia Tech\\
  Blacksburg, VA, 24060 \\
  \texttt{zshuai8@vt.edu} \\
  \And
  Haohui Wang* \\
  Department of Computer Science\\
  Virginia Tech\\
  Blacksburg, VA, 24060 \\
  \texttt{haohuiw@vt.edu} \\
  \AND
  Si Zhang \\
  Meta\\
  Menlo Park, CA, 94025 \\
  \texttt{sizhang@meta.com} \\
  \And
  Dawei Zhou \\
  Department of Computer Science\\
  Virginia Tech\\
  Blacksburg, VA, 24060 \\
  \texttt{zhoud@vt.edu} \\
}

\begin{document}
\maketitle
\footnote{*Equal Contributions}

\begin{abstract}
While graph heterophily has been extensively studied in recent years, a fundamental research question largely remains nascent: \emph{How and to what extent will graph heterophily affect the prediction performance of graph neural networks (GNNs)}? 
In this paper, we aim to demystify the impact of graph heterophily on GNN spectral filters. 
Our theoretical results show that it is essential to design adaptive polynomial filters that adapt different degrees of graph heterophily to guarantee the generalization performance of GNNs. 
Inspired by our theoretical findings, we propose a simple yet powerful GNN named $\name$ by leveraging the MLP-Mixer architectures. 
Our approach comprises two main components: 
(1) an adaptive \emph{patch extractor function} that automatically transforms each node's non-Euclidean graph representations to Euclidean patch representations given different degrees of heterophily, and  
(2) an efficient \emph{patch mixer function} that learns salient node representation from both the local context information and the global positional information. 
Through extensive experiments, the $\name$ model demonstrates outstanding performance on node classification compared with popular homophily GNNs and state-of-the-art heterophily GNNs.
\end{abstract}

\section{Introduction}
Graph Neural Networks (GNNs) have emerged as powerful tools for learning graph-structured data in various domains, such as biology~\cite{huang2020skipgnn}, chemisty~\cite{zhao2021identifying, jiang2021could}, finance~\cite{wang2019semi}, natural language processing~\cite{wu2023graph}, and computer vision~\cite{zhao2019semantic, chen2019multi}. 
From the perspective of graph signal processing, GNNs can be naturally interpreted as a graph filter $g(\mathbf{\Lambda})$ defined in the spectral domain. Formally speaking, given a graph $\mathcal{G} = (\mathcal{V}, \mathcal{E}, \mathbf{X})$ and a set of training node labels $\mathcal{Y}_l$, GNNs are typically designed to learn a spectral filter $g(\mathbf{\Lambda})$ (also referred as hypothesis function) that removes noise from input graph and predict the unknown ``true signal'' $\mathcal{Y}_u$ (i.e., labels of the unlabeled nodes)~\cite{NT2019RevisitingGN}. 
It has been theoretically demonstrated that the majority of GNNs are essentially a low-pass filter~\cite{NT2019RevisitingGN}, which assumes the node labels $\mathcal{Y}$ exhibit a low-frequency nature (i.e., nearby nodes tend to share similar attributes with each other). 
However, it is often the case that real-world graphs are heterophily~\cite{zhu2021graph, zheng2022graph}, where nodes with similar features or common connections have disparate classes. Moreover, growing empirical evidence~\cite{zhu2021graph,zheng2022graph,DBLP:conf/nips/LuanHLZZZCP22} suggests that graph heterophily largely deteriorates conventional GNNs' performance, even worse than vanilla Multilayer Perceptron (MLP)~\cite{zhang2022graphless}.

To alleviate the negative impact of graph heterophily on GNNs, existing works focus on the techniques such as non-local neighbor extension~\cite{zhu2020beyond,abu2019mixhop} and GNN refinement~\cite{liu2023beyond,he2022block}. Despite the promising performance, there are limited theoretical tools to provide a deep understanding of the impact of heterophily on GNNs. Therefore, in this paper, we begin by posing the research question \textbf{Q1: "\textit{How and to what extent will graph heterophily affect the prediction performance of GNNs?}"}. In response, we propose a theoretical model to demystify the complex relationship between the degree of graph heterophily and graph filters $g(\mathbf{\Lambda})$. Perhaps more interestingly, our result implies that the average of frequency responses in $g(\mathbf{\Lambda})$ and the graph heterophily degree may not follow an intuitive linear correlation, which deviates from the common assumption that higher-pass graph filters are preferred for graphs with higher heterophily degrees. As an illustrative example, in Figure~\ref{fig:lambdah}, we present the obtained graph filters $g(\mathbf{\Lambda})$ on various benchmark graphs in a two-dimensional space, where the horizontal axis corresponds to the heterophily degree of input graphs while the vertical axis corresponds to the average of frequency responses in $g(\mathbf{\Lambda})$. 
\begin{wrapfigure}{H}{0.5\textwidth}
    \centering
    \includegraphics[width=\linewidth]{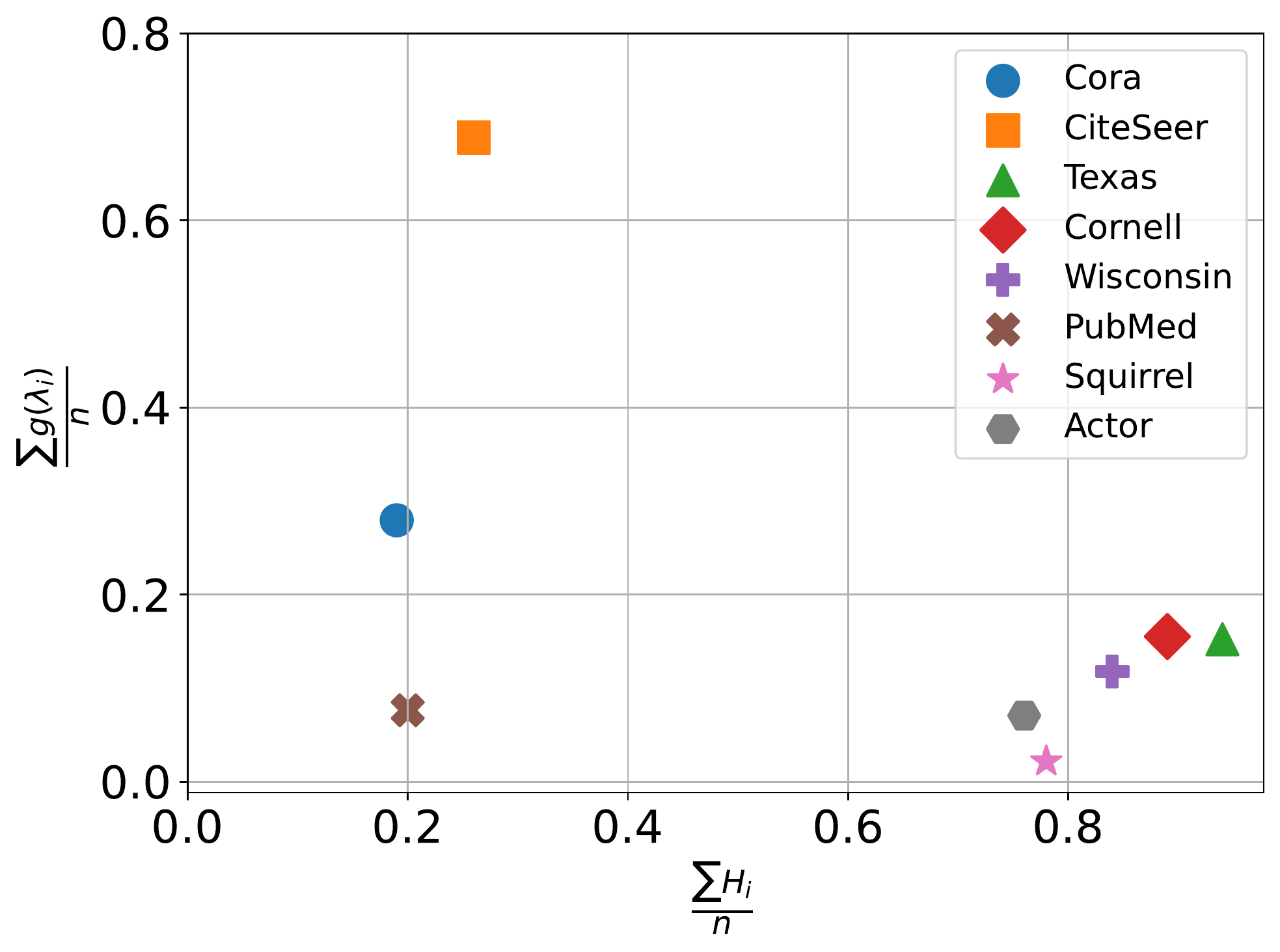}
    \caption{An illustrative example of the average of obtained frequency responses graph filters $\frac{\sum{g(\lambda_i)}}{n}$ vs. heterophily degree $\frac{\sum{H_i}}{n}$ on various benchmark graphs.}
    \label{fig:lambdah}
\end{wrapfigure}

Moreover, our theoretical results motivate us to ask the second research question \textbf{Q2: "\textit{Can we leverage adaptive yet simpler GNN models to automatically characterize the graph heterophily degree and learn salient node representations?}"}. 
Recent studies have illuminated the potentials of MLP-Mixer models in delivering superior performance in computer vision~\cite{liu2023survey} and natural language processing~\cite{lin2022survey}, often comparable to or even exceeding that of the transformer-based and autoregressive models, yet with a simple structure.
This potential leads us to propose a novel MLP-Mixer-based GNN named $\name$, which is designed to learn salient node representations by adapting to the heterophily degree of input graphs. 
While patches in the original MLP-Mixer for images are well-defined, they are not directly applicable to graph data. This is because traditional adjacency matrices do not readily support the conversion from non-Euclidean graph data into a format compatible with Euclidean processing capabilities. To address this issue, we propose a \emph{patch extractor function} that transforms non-Euclidean graph data into Euclidean representations, making them amenable to processing by the $\name$. 
Inspired by our theoretical results, this heterophily-aware patch extractor is also designed to serve as an \emph{adaptive polynomial filters}, capturing individual node importance and incorporating adaptivity based on the degree of heterophily. 
In addition, we introduce an efficient \emph{patch mixer function} that captures both the local context information~\cite{DBLP:conf/www/TangQWZYM15,kipf2017semi} and the global positional information~\cite{you2019position}, which are essential to address graph heterophily across varying degrees.

The key contributions of this paper are threefold.
(1) For the first time, we provide theoretical analyses, which demystify the relationship between graph filter $g(\mathbf{\Lambda})$ and the graph heterophily degree $H$.
(2) Inspired by our theoretical results, we introduce the $\name$, a novel adaptation of the MLP-Mixer model to graphs, maintaining simplicity while adaptively capturing the global positional information and label information based on the degree of heterophily. We further extend it to a more efficient variation named Fast-\name, which embraces more simplicity. 
(3) We substantiate the generalization performance of our $\name$, through rigorous theoretical analysis in Theorem~\ref{theorem:bound} and extensive experiments on both synthetic and real-world networks, emphasizing the potential of simpler yet highly effective architectures for analyzing graph-structured data. 

\section{Preliminaries}
\subsection{Notations}
We use regular letters to denote scalars (e.g., $c$), boldface lowercase letters to denote vectors (e.g., $\mathbf{r}$), and boldface uppercase letters to denote matrices (e.g., $\mathbf{X}$).
We represent a graph as $\mathcal{G} = (\mathcal{V}, \mathcal{E}, \mathbf{X})$, where $\mathcal{V}$ represent the set of nodes, $\mathcal{E} \subseteq \mathcal{V} \times \mathcal{V}$ represent the set of edges, and $\mathbf{X}$ represent the node feature matrix. We define $n$ as the number of nodes, $\mathcal{Y}$ as the set of labels. $\mathbf{L}$ ($\mathbf{A}$) is the Laplacian (adjacency) matrix, and $\tilde{\mathbf{L}}$ ($\tilde{\mathbf{A}}$) is the normalized Laplacian (adjacency) matrix. Table~\ref{tab:notation} in Appendix~\ref{sec:notation} summarizes the main symbols and notations in this paper.

\subsection{Graph Spectral Filters}
Spectral-based GNNs provide an effective way of spectral analysis on graphs according to graph signal processing. Specifically, the eigendecomposition of symmetric normalized Laplacian is $\tilde{\mathbf{L}} = \mathbf{U}\mathbf{\Lambda} \mathbf{U}^T$, where $\mathbf{U}\in\mathbb{R}^{n\times n}$ is the graph Fourier basis formed by $n$ eigenvectors, and $\mathbf{\Lambda}$ denotes the diagonal eigenvalue matrix. Then the graph Fourier transform of the input signal of all nodes $\textbf{x}\in\mathbb{R}^n$ is defined as $\hat{\textbf{x}}=\mathbf{U}^T \textbf{x}$, where $\hat{\textbf{x}}$ is the Fourier transformed graph signal. The inverse graph Fourier transform is given as $\textbf{x}=\mathbf{U}\hat{\textbf{x}}$. Hence, the spectral graph convolution of the signal $\textbf{x}$ is defined as $\textbf{x}*_{\mathcal{G}} g = \mathbf{U}g(\mathbf{\Lambda}) \mathbf{U}^T \textbf{x}$, where $g$ is the graph filter. 
Vanilla graph convolutional network (GCN)~\cite{kipf2017semi} can be considered as the spectral convolution with a simplified first-order Chebyshev polynomial to approximate the spectral filter $g(\mathbf{\Lambda})=(I+\mathbf{\Lambda})^{-1}$, which attenuates the high-frequency more than the low-frequency thus is a low-pass filter~\cite{NT2019RevisitingGN,wu2019simplifying}.

\subsection{Graph Heterophily}
Many of the existing works on GNN models share a vital assumption of graph homophily, i.e., locally connected nodes should be similar in terms of their features and labels~\cite{kipf2017semi, wu2019simplifying}. However, recent research shows interest in the limitations of graph models in dealing with heterophily graphs, where connections often occur between nodes that are not similar. A commonly used metric for modeling the heterophily of a graph is node heterophily~\cite{zheng2022graph}.
\begin{definition}\label{def:nodeH}
The neighbor set of $i$-th node $v_i$ is denoted as $\mathcal{N}(v_i)=\{v_j: (v_i, v_j)\in\mathcal{E}\}$, $y_i$ is the label of node $v_i$, then we define the node heterophily $H_i$ as the proportion of neighbors with a different label of node $v_i$:
\begin{equation}
    H_i = \frac{|\{v_j\in\mathcal{N}(v_i): y_j \neq y_i\}|}{|\mathcal{N}(v_i)|}.
\end{equation}
$H=(H_0,\ldots,H_{n-1})$ defines vector of all node heterophilies in graph $\mathcal{G}$.
\end{definition}
Based on the above definition of node-level heterophily, the graph-level heterophily can be derived by averaging across nodes, i.e., $\frac{1}{n}\sum_{0\leq i\leq n-1}H_i$.  
Real-world networks, such as those represented by the Texas, Squirrel, Chameleon, Cornell, Wisconsin, and Actor datasets~\cite{rozemberczki2021multi,10.1145/1557019.1557108,garcia2016using}, often exhibit a high heterophily. The standard approach of graph convolutional networks, which updates node features by aggregating neighboring node features, is thus less effective in these cases. 
Therefore, it is crucial to gain a mathematical understanding of the research question (\textbf{Q1}).
\section{Demystify the Impact of Graph Heterophily on Graph filters}

Existing works \cite{zhu2020beyond,liu2023beyond,li2021beyond} share a common consensus that high-pass graph filters are preferable for heterophily graphs, while low-pass filters are favored for homophily graphs. However, this consensus is weak and not necessarily accurate. The study of $g(\mathbf{\Lambda})$ focuses on low-pass/ high-pass filters in the spectral domain, while heterophily indicates the differences among neighbor labels, which, although highly correlated, are not necessarily equivalent. For example, consider a graph where each node is assigned a different label, resulting in a highly heterophilic graph. In this case, a  high-pass filter may be appropriate for dealing with the graph heterophily. However, in the graph coloring problem where each node is assigned a color such that connected nodes have different colors while minimizing the total number of colors used, the graph exhibits extreme heterophily, with all neighbor nodes having different labels (colors). Nevertheless, since the number of labels (colors) may be limited, only a low-pass filter may be required to capture the necessary information~\cite{li2022rethinking,wang2023graph}. 

A key challenge to building the connection between graph filter and graph heterophily is that they are originally defined in different domains: graph filter is given in the spectral domain, while graph heterophily is in the spatial domain. Therefore, we propose a novel notion of graph heterophily that is defined in the spectral domain via graph Fourier Transformation.
\begin{definition}
Denote $\mathbf{U}$ as the eigenvector matrix of normalized adjacency matrix $\tilde{\mathbf{A}}$, the heterophily in spectral domain $\hat{H}$ is defined as:
\begin{equation}
    \hat{H}  = H \cdot \mathbf{U}
\end{equation}
where $\hat{H} = (\hat{H}_0,\ldots,\hat{H}_{n-1})$ is denoted as the frequency components of $H$. 
\end{definition}

Based on Definition 2, we provide the following lemma to show the lower bound of graph filters concerning graph heterophily. The proof of Lemma~\ref{lemma:filter} is provided in Appendix~\ref{sec:proof}.
\begin{restatable}{lemma}{lemmafilter}
\label{lemma:filter}
$\tilde{\mathbf{A}}$ is the normalized adjacency matrix, $n$ is the number of nodes, $\lambda_i$ is the $i$-th eigenvalue of $\tilde{\mathbf{A}}$, where $-1=\lambda_0\leq\cdots\leq\lambda_{n-1}\leq 1$, and $\mathbf{\Lambda}$ is a diagonal matrix of all eigenvalues. $g(\mathbf{\Lambda})$ is the graph filter, then we have the lower bound of the average of $g(\mathbf{\Lambda})$:
\begin{equation}
    \sum_{0\leq i\leq n-1} \frac{g\left(\lambda_i\right)}{n} \geq \frac{\sum_{0\leq i\leq n-1} \log |\hat{H}_i|}{n(\log \sum_{0\leq i\leq n-1} g\left(\lambda_i\right) |\hat{H}_i|-\log \sum_{0\leq i\leq n-1} g\left(\lambda_i\right))}
\end{equation}
\end{restatable}

\textbf{Remark.} Lemma~\ref{lemma:filter} theoretically demonstrates that the average of frequency responses in $g(\mathbf{\Lambda})$ and the graph heterophily degree might not follow a linear correlation. The general observation is that heterophily graphs have a smaller average frequency response, and homophily graphs have a larger average frequency response. However, as shown in Figure~\ref{fig:lambdah}, PubMed dataset has a small heterogeneity degree, but the average of frequency responses graph filters obtained from our experiments is small unlike other datasets, illustrating that $\frac{1}{n}\sum_{0\leq i\leq n-1} g(\lambda_i)$ and $\frac{1}{n}\sum_{0\leq i\leq n-1} \hat{H}_i$ may follow a complex nonlinear relationship.

\section{Approach}

In this section, we first introduce the two main components of our model $\name$: patch extractor and  patch mixer, discussed in detail in Section 4.1. These components collaboratively resolve the challenges of graph heterophily for semi-supervised node classification. Then, we delve into the theoretical aspects of our approach, elucidating the generalization bound to ensure the robustness of our model. Furthermore, we propose a fast algorithm in Section 4.2 to improve the efficiency and scalability of our model. This comprehensive discussion underlines the unique characteristics of our approach, distinguishing it from classic GNN methodologies and those tailored for graph heterophily.

\subsection{$\name$ Framework}
Here we present the technical details of the patch extractor, which extracts patches from $\mathcal{G}$, followed by the patch mixer, which mixes information within each patch to learn better node representations.
\paragraph{Patch Extractor.}
The patch extractor acts as the pivotal module of our model, transforming the graph-structured data into a Euclidean representation. Its primary function is to identify and extract a relevant set of top-$p$ nodes tailored to a specific objective. To address the heterophily challenge \textbf{Q2}, we introduce a heterophily-aware patch extractor that effectively recognizes and adapts to varying heterophily levels. This extractor constructs a patch for each individual node, which is a subset of nodes (not necessarily the direct neighbors) that hold high relevance in the context of graph heterophily. The extractor selects nodes with the highest relevance scores and incorporates the top-$p$ scoring nodes into the corresponding node's patch. Grounded in our theoretical investigation of graph heterophily and guided by \textbf{Q2}, the patch extractor's design capitalizes on the graph's spectral properties and exhibits the necessary adaptability to accommodate diverse heterophily degrees. Consequently, the generated patches capture the core characteristics of the graph data and its heterophily degree, forming a solid basis for the subsequent stages of our model.

The heterophily-aware patch extractor utilizes the normalized adjacency matrix $\tilde{\mathbf{A}}$ instead of the graph Laplacian, as the former provides a powerful representation of local graph structure and its underlying relationships. This choice allows the spectral graph filter framework to be utilized, emphasizing local connectivity and effectively addressing graph heterophily. 
The filter operation based on an adaptive polynomial filter is defined as 
\begin{equation}
g(\mathbf{\Lambda}) = \sum^{K}_{k=1}\sigma(\textbf{w}_k\odot\mathbf{\Lambda}^k),
\end{equation}
where $\sigma$ denotes a non-linear activation function, and $\textbf{w}_k$ are learnable parameters. To extract the patch for each node, we first compute a scoring matrix $\textbf{R} \in \mathbb{R}^{n \times n}$ where 
\begin{equation}
    \textbf{R} = \mathbf{U} \left(\sum_{k=1}^K \textbf{w}_k\odot\mathbf{\Lambda}^k\right)\mathbf{U}^T.
\end{equation}
Subsequently, we utilize it as a rank-order metric to extract the top-$p$ nodes, viewed as the patch for each node. For each column $col$ in $\textbf{R}$, the indices are organized in descending order according to their corresponding relevance scores. This operation is mathematically represented as $\text{sorted}_{col} = \text{argsort}(\textbf{R}{:,col}, \text{descending})$. Following this, the indices of the top $p$ nodes with the highest relevance scores for each column $col$ are chosen, denoted as $\text{top-}p_{col} = \text{sorted}_{col}[0:p]$. In this way, the patch extractor is equivalent to
\begin{equation}
\phi(\tilde{\mathbf{A}}) = \text{top-}p_{col}\biggl(\mathbf{U} \left(\sum_{k=1}^K \textbf{w}_k\odot\mathbf{\Lambda}^k\right) \mathbf{U}^T \biggl).
\end{equation}

Using the selected node indices, we extract input node features of the patch for each node $v$, denoted by $\textbf{P}_v = { \textbf{X}[\text{top-}p(\textbf{R}_v)] }$. This process results in patches $\textbf{P} \in \mathbb{R}^{n \times p \times d}$. This patch representation encapsulates the learning of complex relationships among nodes. It also allows for an effective approach to handle graph heterophily via the patch mixer, thereby enhancing the precision of node representation and improving model performance across various levels of graph heterophily.
\paragraph{Alternate Patch Extractor}
To underscore the adaptability and versatility of our model, we consider several alternative patch extractors, including heat filters, bandpass filters~\cite{bahonar2019graph}, and a method employing shared parameters. The heat filter, emulating a graph's diffusion process, works in the spectral domain. Given a normalized adjacency matrix $\mathbf{\tilde{A}}$ with eigenvalues $\mathbf{\lambda}_i$ and eigenvectors $\mathbf{u}_i$, the basis of heat filter extractor is represented as $g_{(heat)}(\mathbf{\Lambda}, t) = \sum^{K}_{k=1}\sigma(e^{-t\mathbf{\Lambda}}$), with $t$ serving as the filter number in our model. Conversely, the bandpass filter isolates specific spectral components of a graph signal. For a normalized adjacency matrix $\tilde{\mathbf{A}}$ with corresponding eigenvalues and eigenvectors, the basis of the bandpass filter extractor is defined as $g_{(bandpass)}(\mathbf{\Lambda}, l, h) = \sum^{K}_{k=1}\sigma(\mathcal{R}(\mathbf{\Lambda}, l, h, k))$, where $l$ and $h$ set the lower and upper-frequency boundaries, respectively, and $\mathcal{R}(\mathbf{\Lambda}, l, h, k)$ is a rectangular function. The concept of shared parameters refers to a design choice where a common set of filter parameters is employed across all filter numbers. Formally, the basis of the share parameters extractor is defined as $g_{(shared)}(\mathbf{\Lambda}) = \sum^{K}_{k=1}\sigma(\textbf{w}_1\mathbf{\Lambda}^k)$. This approach enhances model parsimony and consistency in the patch extraction process, potentially leading to more powerful representations.
\paragraph{Patch Mixer.}
The original MLP-Mixer consists of two types of layers: channel-mixing and patch-mixing layers. In response to \textbf{Q1}, we redefine these concepts for graph-structured data as follows:
\begin{enumerate}
\item \textbf{Feature-Mixing Layer:} Given a graph $\mathcal{G} = (\mathcal{V}, \mathcal{E})$ with nodes $\mathcal{V}$ and edges $\mathcal{E}$, and a patch $\mathbf{P}_v$ for a node $v \in \mathcal{V}$, this layer applies a multilayer perceptron (MLP) to the nodes in $\mathbf{P}_v$ to mix information within the patch, denoted as $\mathbf{P}_v = \mathrm{MLP}(\mathbf{P}_v)$.

\item \textbf{Patch-Mixing Layer:} This layer mixes information across patches. Given the output $\mathbf{P}_v$ of the patch-mixing layer for each node $v \in \mathcal{V}$, it applies an MLP across patches independently for each feature, denoted as $\mathbf{P} = \mathrm{MLP}(\mathbf{P}^T)^T$, where $\mathbf{P}$ is the collection of all patches $\mathbf{P}_v$ for all nodes $v \in \mathcal{V}$.
\end{enumerate}

In our patch mixer, we incorporate two distinct types of information: local contextual information~\cite{kipf2017semi,DBLP:conf/www/TangQWZYM15} and global positional information~\cite{you2019position}. The local contextual information encapsulates the attribute data related to a node and its immediate topological neighborhood. This localized data is captured and processed by the Feature-Mixing Layer, which applies a multilayer perceptron (MLP) to nodes within each patch, effectively mixing the information within these local patches. On the other hand, the global positional information, which refers to a node's broader position and role within the entire graph, is incorporated by the Patch-Mixing Layer. This layer employs an MLP across all patches for each feature, allowing for a global perspective that considers the interconnectedness and relationships between patches. Thus, our model adeptly integrates both local contextual information and global positional information via the feature-mixing layers and patch-mixing layers. 

\paragraph{Generalization Analysis of \name. }
Although we propose patches that adaptively capture the varying degrees of heterophily, the extent to which our algorithm can guarantee generalization performance remains unclear because of the limited theoretical tools. Therefore, in order to shed light on this issue, we propose an upper bound of the error in terms of graph filters, graph heterogeneity, graph features, and labels. The proof of Theorem~\ref{theorem:bound} is provided in Appendix~\ref{sec:proof}.
\begin{restatable}{theorem}{theorembound}
\label{theorem:bound}
For a binary classification problem on graph $\mathcal{G}_{n}$ with $\tilde{\mathbf{A}}$, let $\mathbf{X}=\left(\mathbf{x}_{0}, \mathbf{x}_{1}\right)$ be the learnable input of $g(\tilde{\mathbf{A}})$ in the last layer and $\mathbf{Y}=\left(\mathbf{y}_{0}, \mathbf{y}_{1}\right)$ be the label matrix, the error bound is:
\begin{equation}
    \overline{Er}(\mathbf{X}, \mathbf{Y}) \leq c_1 - \frac{\min_{i \in \mathcal{I}_{g, \delta, \eta}}\psi_{\frac{1}{g\left(1-\mathbf{\lambda}_i\right) \delta_i}}\left(\eta_i\right) \cdot \delta_i \sum_{i \in \mathcal{I}_{\delta, \tilde{\eta}}} \log |\hat{H}_i|}{2n\log \sum_{i \in \mathcal{I}_{\delta, \tilde{\eta}}} g\left(1-\mathbf{\lambda}_i\right) |\hat{H}_i| - 2n\log \sum_{i \in \mathcal{I}_{\delta, \tilde{\eta}}} g\left(1-\mathbf{\lambda}_i\right)}
\end{equation}
where $c_1$ is a constant, $\delta$, $\eta$ are the spectra of $\Delta y$ and $\Delta x$, $\mathcal{I}_{\delta}=\{i|\delta_i \neq 0, i=0, \cdots, n-1\}$ is the indicator set of nonzero elements of $\delta$, and $\psi(x)=\min\{\max\{x,-1\},1\}$ is a clamp finction.
\end{restatable}

\subsection{Fast-\name}
While the heterophily-aware patch function is a powerful method for adaptively capturing graph heterophily, its computational complexity can be the bottleneck for large-scale graph datasets due to the computations of eigen decomposition (i.e., $O(n^3)$). To address this challenge, we seek to develop an efficient algorithm that retains the performance of our model without compromising scalability and applicability. In pursuit of an efficient patch extraction method, we uncover the relationship between our adaptive polynomial filters-based method and the non-adaptive variant inspired by a classic optimization-based ranking algorithm which aims to minimize the following objective function for each node $v$.
\begin{equation}
\label{eq:ppr}
    \mathcal{J}(\mathbf{r}_v)=c\mathbf{r}_v^T(\mathbf{I}-\tilde{\mathbf{A}})\mathbf{r}_v+(1-c)\|\mathbf{r}_v-\mathbf{e}_v\|_2^2
\end{equation}
where $\mathbf{r}_v$ is the ranking vector w.r.t. node $v$ and $\mathbf{e}_v$ is a node indicator vector. To solve this optimization problem, the iterative approach can be formulated as 
\begin{equation}
\textbf{r}_v = c \tilde{\mathbf{A}}\textbf{r}_v + (1-c) \textbf{e}_v
\end{equation}
and the closed-form solution is $\textbf{r}_v = (1 - c)(1 - c \tilde{\mathbf{A}})^{-1} \textbf{e}_v$.



By the Neumann series expansion, the closed-form solution can be approximated by
\begin{equation}
\textbf{r}_v = (1 - c) \sum_{k=0}^{K} c^k \tilde{\mathbf{A}}^k \mathbf{e}_v = (1 - c) \sum_{k=0}^{K} c^k (\mathbf{U}\mathbf{\Lambda}^k \mathbf{U}^T) \textbf{e}_v.
\end{equation}
Obviously, by selecting top-$p$ nodes based on $\mathbf{r}_v$, it can be considered as a special case of Eq. (6) by substituting $(c\mathbf{\Lambda})$ with $\textbf{w}\mathbf{\Lambda}$ 
\begin{equation}
\phi_{\text{fast}}(\tilde{\mathbf{A}}) = \text{top-}p_{col}(\mathbf{R})=\text{top-}p_{col}([\mathbf{r}_1, \cdots, \mathbf{r}_n])
\end{equation}
Compared with heterophily-aware patch extractor's $O(n^3)$ complexity,  Fast-$\name$ is more scalable and computation efficient with $O((n + m)t)$ where $m$ is the number of edges and $t$ is the number of iterative steps. 
Fast-$\name$ extractor computes proximity scores in an iterative manner. In each iteration, the algorithm must update the rank value for each node in the graph, which is a process with a time complexity of $O(n + m)$,
\section{Experiment}
\subsection{Experiment Setup}
\paragraph{Datasets.} We evaluate our model, GraphPatcher, using several standard graph datasets originated from~\cite{fout2017protein}: Cora, Citeseer, PubMed, Texas, Squirrel, Chameleon, Cornell, Wisconsin, and Actor. These datasets are diverse in terms of their sizes and degrees of heterophily. They cover a range of domains including citation networks (Cora, Citeseer, PubMed\cite{sen2008collective, namata2012query}), web page networks (Texas, Cornell, Wisconsin\cite{garcia2016using}), animal interaction networks (Squirrel, Chameleon\cite{rozemberczki2021multi}), and Actor is an actor co-occurrence network\cite{10.1145/1557019.1557108}. Homophily datasets are Cora, Citeseer, and PubMed, while Texas, Squirrel, Chameleon, Cornell, Wisconsin, and Actor are heterophily datasets. The statistics of these datasets are summarized in Appendix~\ref{sec:data}.
    
\paragraph{Comparison Methods.}We compare our model with several baseline methods, including GPRGNN~\cite{chien2021adaptive}, ChebNet~\cite{tang2019chebnet}, APPNP~\cite{DBLP:conf/iclr/KlicperaBG19}, GCN-JKNet~\cite{DBLP:conf/icml/XuLTSKJ18}, GCN~\cite{kipf2017semi}, GAT~\cite{velickovic2018graph}, GraphSage~\cite{hamilton2017inductive}, FAGCN~\cite{li2021beyond}, H2GCN~\cite{zhu2020beyond}, and MLP. GPRGNN aims to improve the performance of GNNs by incorporating graph pooling and refinement. ChebNet utilizes Chebyshev polynomial filters to perform graph convolutions. APPNP introduces personalized PageRank to iteratively propagate information in the graph. GCNJKNet combines graph convolutional networks (GCN) and Jumping Knowledge Networks (JKNet) to capture multi-scale information. GCN is a classic graph convolutional network that operates on a fixed neighborhood. GAT employs attention mechanisms to selectively aggregate information from neighboring nodes. GraphSage performs inductive learning by sampling and aggregating features from a node's local neighborhood. FAGCN incorporates both low and high-frequency signals. H2GCN focuses on analyzing heterophilic graphs by considering higher-order neighborhoods. MLP is an attribute-only-based multi-layer perceptron. We compare our proposed model against these baselines to evaluate its performance and effectiveness in handling graph heterophily.

\begin{table*}[!tbp]
    \centering
    \vspace{-5mm}
    \resizebox{\textwidth}{!}{%
        \begin{tabular}{l|ccc|ccccccc}
        \hline
        & \multicolumn{3}{c|}{Homophily Dataset} & \multicolumn{6}{c}{Heterophily Dataset} \\ \cline{2-10}
        Model & Cora & CiteSeer & PubMed & Texas & Squirrel & Chameleon & Cornell & Wisconsin & Actor \\ \hline
        $\frac{\sum{H_i}}{n}$ & 0.19 & 0.26 & 0.2 & 0.89 & 0.78 & 0.77 & 0.89 & 0.84 & 0.76 \\ \hline
        MLP & 49.7 $\pm$ 1.5 & 48.8 $\pm$ 4.6 & 69.8 $\pm$ 1.1 & 74.4 $\pm$ 5.6 & 33.6 $\pm$ 1.4 & 45.8 $\pm$ 2.0 & 69.9 $\pm$ 2.9 & 80.4 $\pm$ 3.7 & 35.0 $\pm$ 0.7 \\ \hline
        GPRGNN & 76.1 $\pm$ 0.9 & 63.9 $\pm$ 1.2 & 75.6 $\pm$ 0.7 & 64.2 $\pm$ 5.3 & 30.4 $\pm$ 2.0 & 36.3 $\pm$ 2.1 & 59.1 $\pm$ 4.4 & 72.8 $\pm$ 4.5 & 30.9 $\pm$ 0.7 \\ \hline
        ChebNet & 73.7 $\pm$ 0.7 & 59.7 $\pm$ 1.1 & 75.1 $\pm$ 0.8 & 74.4 $\pm$ 5.8 & 33.4 $\pm$ 1.0 & 47.5 $\pm$ 1.5 & 69.7 $\pm$ 2.8 & 77.9 $\pm$ 3.0 & 34.8 $\pm$ 1.0 \\ \hline
        APPNP & 76.5 $\pm$ 1.6 & 64.2 $\pm$ 1.3 & 75.8 $\pm$ 1.7 & 56.3 $\pm$ 3.9 & 27.8 $\pm$ 0.7 & 38.8 $\pm$ 1.5 & 42.9 $\pm$ 3.7 & 52.3 $\pm$ 4.9 & 28.4 $\pm$ 0.6 \\ \hline
        GCNJKNet & 76.9 $\pm$ 1.7 & 63.4 $\pm$ 1.4 & 74.1 $\pm$ 0.9 & 57.8 $\pm$ 2.0 & 26.3 $\pm$ 0.6 & 37.5 $\pm$ 1.9 & 40.9 $\pm$ 4.6 & 49.6 $\pm$ 4.8 & 27.9 $\pm$ 0.8 \\ \hline
        GCN & 76.1 $\pm$ 2.3 & 63.9 $\pm$ 2.4 & 75.8 $\pm$ 0.9 & 58.7 $\pm$ 2.8 & 27.1 $\pm$ 0.5 & 39.9 $\pm$ 2.0 & 40.3 $\pm$ 3.3 & 49.4 $\pm$ 3.0 & 28.4 $\pm$ 0.7 \\ \hline
        GAT & 75.6 $\pm$ 2.3 & 63.4 $\pm$ 1.0 & 75.2 $\pm$ 1.0 & 58.3 $\pm$ 4.8 & 28.7 $\pm$ 0.9 & 42.8 $\pm$ 1.5 & 46.3 $\pm$ 4.3 & 51.1 $\pm$ 6.2 & 28.9 $\pm$ 0.5 \\ \hline
        GraphSage & 73.9 $\pm$ 0.5 & 63.7 $\pm$ 1.6 & 75.5 $\pm$ 1.1 & 74.3 $\pm$ 3.7 & 36.2 $\pm$ 0.8 & 47.4 $\pm$ 1.8 & 69.1 $\pm$ 3.5 & 76.9 $\pm$ 4.5 & 34.6 $\pm$ 0.7 \\ \hline
        FAGCN & 73.2 $\pm$ 0.7 & 61.2 $\pm$ 0.7 & 74.9 $\pm$ 1.1 & 65.4 $\pm$ 2.8 & 33.9 $\pm$ 0.8 & 43.2 $\pm$ 0.9 & 55.4 $\pm$ 5.1 & 65.7 $\pm$ 4.3 & 34.4 $\pm$ 0.5 \\ \hline
        H2GCN & 76.5 $\pm$ 2.3 & 60.5 $\pm$ 0.9 & 75.5 $\pm$ 1.9 & 81.1 $\pm$ 5.2 & 32.2 $\pm$ 1.2 & 54.0 $\pm$ 1.1 & 68.5 $\pm$ 2.9 & 78.0 $\pm$ 3.1 & 34.9 $\pm$ 0.4 \\ \hline
        BM-GCN & 75.6 $\pm$ 2.3 & 63.0 $\pm$ 1.0 & 75.2 $\pm$ 1.0 & 75.4 $\pm$ 4.8 & 34.2 $\pm$ 0.9 & 53.6 $\pm$ 1.5 & 62.3 $\pm$ 4.3 & 75.4 $\pm$ 5.2 & 34.7 $\pm$ 0.5 \\ \hline
        \midrule
         Fast-$\name$ & 76.3 $\pm$ 0.6 & 62.5 $\pm$ 1.8 & 75.3 $\pm$ 0.9 & 81.3 $\pm$ 2.2 & 36.7 $\pm$ 1.1 & 55.4 $\pm$ 1.1 & 69.9 $\pm$ 2.3 & 80.5 $\pm$ 2.4 & 36.7 $\pm$ 0.9 \\ \hline
         $\name$ & \textbf{77.3} $\pm$ 0.6 & \textbf{64.5} $\pm$ 1.8 & \textbf{76.3} $\pm$ 0.9 & \textbf{85.0} $\pm$ 3.5 & \textbf{37.8} $\pm$ 1.6 & \textbf{57.3} $\pm$ 1.2 & \textbf{71.5} $\pm$ 3.4 & \textbf{83.2} $\pm$ 3.4 & \textbf{37.8} $\pm$ 0.7 \\ \hline
    \end{tabular}%
    }
    \caption{Results for different models on the homophily and heterophily datasets.}
    \vspace{-7mm}
    \label{tab:overall performance}
\end{table*}
\paragraph{Implementation Details.} We employ the Adam optimizer with a learning rate of 0.005 and a weight decay of 5e-4 for training the model. The default training setting for all models is performed using torch-geometric default masks(train, validation, and test) and trained for a maximum of 500 epochs, a hidden dimension of 64, a dropout rate of 0.5, and a number of layers of 2. Early stopping is applied with a patience of 50 epochs, which monitors the validation loss and terminates the training if there is no improvement observed within the specified patience. This helps prevent overfitting and ensures the model generalizes well to unseen data. To assess the performance of the models, we use the standard accuracy metric, which is commonly adopted in node classification tasks. 
\begin{wraptable}[8]{r}{0.5\textwidth}
\centering
\small
\begin{tabular}{l|c|c}
\midrule
\textbf{Model} & \textbf{Cora}   & \textbf{Chameleon} \\
\hline
GCN   & 74.1 $\pm$ 1.0 & 56.0 $\pm$ 2.0\\
FAGCN & 74.8 $\pm$ 0.4 & 48.3 $\pm$ 2.2 \\
\name & \textbf{77.3} $\pm$ 0.6 & \textbf{57.3} $\pm$ 1.2 \\
\hline
\end{tabular}
\caption{Ablation study results comparing GCN, FAGCN on Cora and Chameleon datasets with induced graphs from heterophile-aware patch function.}
\label{tab:induced}
\end{wraptable}
\subsection{Quantitive Analysis}
In this section, we present a comprehensive quantitative analysis comparing the performance of our proposed method with various state-of-the-art graph neural network models. Our proposed method consistently outperforms the other models across both homophily and heterophily datasets. 
In particular, the performance gains are more prominent in the challenging heterophily datasets, where the underlying graph structure exhibits more complex relationships between nodes. The results demonstrate the effectiveness of our approach in capturing and leveraging the diverse connectivity patterns present in these datasets. 
We also examine different variations of our proposed method, including the fast $\name$ and standard $\name$ model. The performances of these variations are relatively similar, indicating the robustness and versatility of our method in handling different graph structures. In summary, the quantitative analysis demonstrates the superiority of our approach compared to the existing models, especially in the context of heterophily datasets.

\paragraph{Ablation Study: Induced Graph.}We conduct an ablation study comparing the performance using several other popular graph neural network models on induced graphs obtained from our patch function. The baseline models include GCN, H2GCN, and the evaluation datasets used include Cora and Chameleon.

We summarize the results of the ablation study in Table \ref{tab:induced}, which includes the performance metrics of each model on the dataset. The results in Table \ref{tab:induced} demonstrate that our method achieves better performance compared to the other models, highlighting the effectiveness of our heterophily-aware patch function and the adaptive polynomial filters. Additionally, the comparison between H2GCN and FAGCN reveals the importance of incorporating adaptive polynomial filters in addressing graph heterophily.
\begin{wrapfigure}[13]{r}{0.4\linewidth}
    \centering
    \setlength{\belowcaptionskip}{-40pt}
    \includegraphics[width=\linewidth]{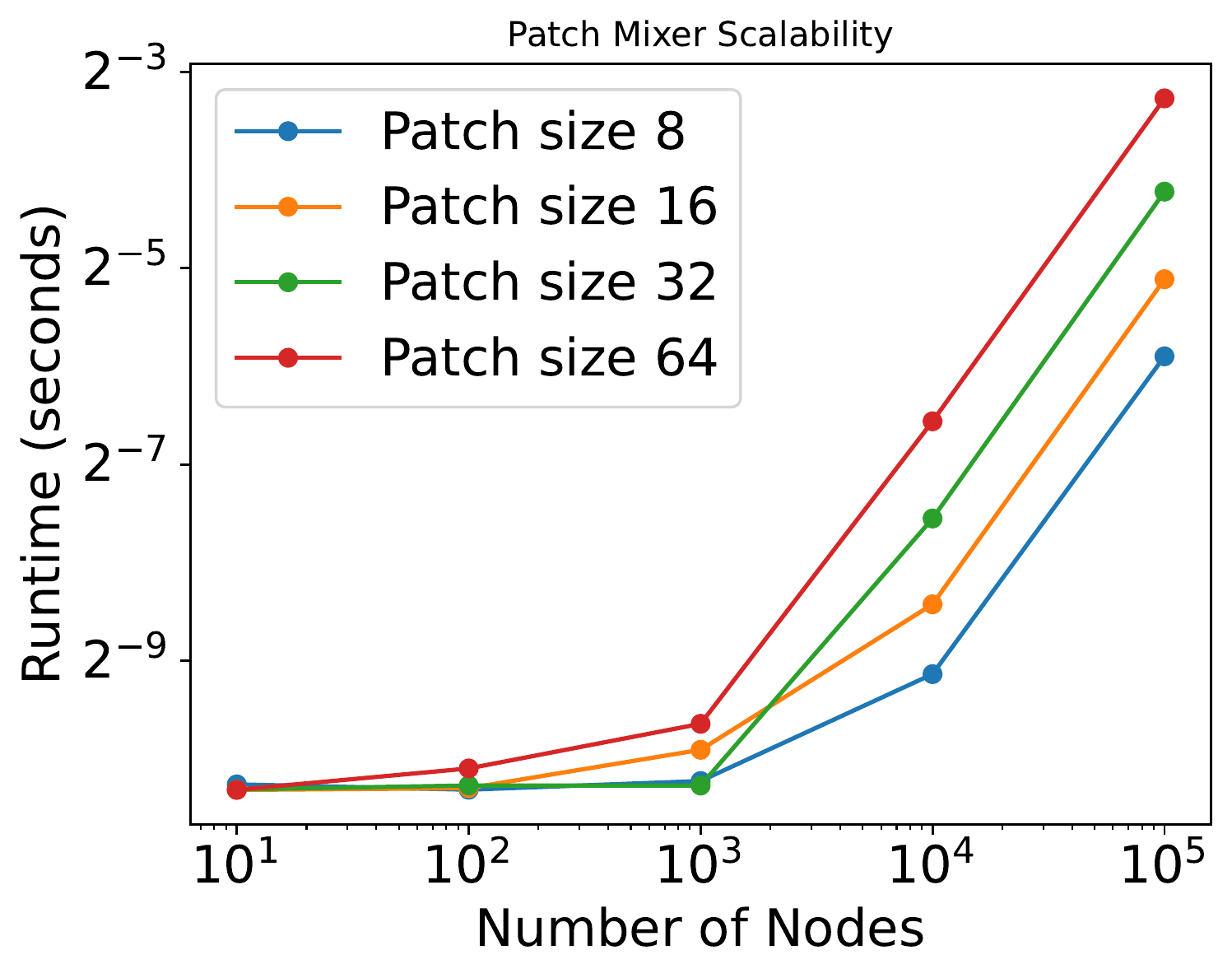}
    \caption{Scalability analysis.}
    \label{fig:label_perm}
\end{wrapfigure}
\begin{wraptable}[9]{r}{0.5\textwidth}
\centering
\small
\begin{tabular}{l|c|c}
\midrule
\textbf{Model} & \textbf{Cora} & \textbf{Chameleon} \\
\hline
Bandpass Filter   & 71.3 $\pm$ 0.8  & 52.1 $\pm$ 1.5 \\
Heat Filter   & 70.4 $\pm$ 1.69  & 47.7 $\pm$ 1.73 \\
Shared parameters & 73.0 $\pm$ 1.3  & 52.2 $\pm$ 1.1 \\
Heterophily-aware  & \textbf{77.3} $\pm$ 0.6  & \textbf{57.3} $\pm$ 1.2 \\
\hline
\end{tabular}
\caption{Performance comparison of alternative patch functions on Cora and Chameleon datasets.}
\label{tab:alternative_patch_functions}
\end{wraptable}
\paragraph{Ablation Study: Alternative Patch Functions.}
We assess the impact of alternative parameters in our method by comparing the performance of our model using different patch functions and settings.

The results in Table \ref{tab:alternative_patch_functions} indicate that the choice of patch function and parameter settings plays a significant role in the performance of our model. Using a shared parameter setting yields better results across the datasets, suggesting that it effectively captures the underlying graph structure and accurately represents the relationships among nodes in the graph. Moreover, the comparison between bandpass and heat filters reveals the importance of selecting an appropriate filter function for the given task and dataset.

\begin{figure}
\centering
\includegraphics[width=0.9\textwidth]{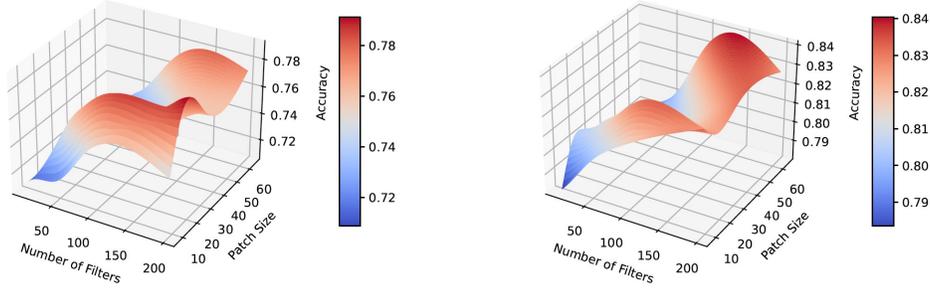}
\caption{Parameter sensitivity analysis results for Cora(left) and Texas(right) datasets, respectively.}
\label{fig:param_sensitivity}

\end{figure}
\paragraph{Ablation Study: Patch Order Importance.}
We investigate the importance of patch order in our method at table~\ref{tab:patch_order_importance}. We compare the performance of our model with different patch orders by random shuffling, where the patch order refers to the number of neighbor nodes considered in the patch extraction process. Our findings reveal that shuffling the patch order leads to significant degrading performance, indicating that considering patch node order plays a vital role in capturing the complex graph structures and addressing heterophily.
\paragraph{Scalability Analysis on patch mixer.}
We analyze the runtime of the algorithm as the number of nodes and patch size varies to assess the scalability of our proposed approach in Figure~\ref{fig:label_perm}. We perform experiments on synthetic datasets with varying numbers of nodes from \{10, 100, 1,000, 10,000, 100,000\}. We also consider different patch sizes, specifically \{8, 16, 32, 64\}. The runtime results are reshaped into a matrix for visualization purposes. Figure~\ref{fig:label_perm} shows the runtime in seconds on a logarithmic scale for the different patch sizes as a function of the number of nodes. As expected, the runtime increases with the number of nodes and patch size. Nevertheless, our method demonstrates reasonable scalability as the growth in runtime is sub-linear, indicating that our approach can handle large-scale graph datasets effectively. The results also highlight the importance of selecting an appropriate patch size to balance the trade-off between computational efficiency and performance.

\begin{wraptable}[8]{r}{0.5\linewidth}
    \centering
    \small
    \begin{tabular}{l|c|c}
        \midrule
        \textbf{Dataset} & \textbf{Random Order} & \textbf{Ranked Order} \\
        \hline
        Cora       & $71.4 \pm 1.1$ & $\textbf{77.3} \pm 0.6$ \\
        Citeseer   & $60.7 \pm 1.4$  & $\textbf{64.5} \pm 1.8$ \\
        Squirrel   & $35.4 \pm 1.2$ & $\textbf{37.8} \pm 1.6$ \\
        Chameleon  & $51.0 \pm 1.2$ & $\textbf{57.3} \pm 1.1$ \\
        \hline
    \end{tabular}
    \caption{Importance of patch order.}
\label{tab:patch_order_importance}
\end{wraptable}

\paragraph{Parameter Sensitivity Analysis.}
We conduct a parameter sensitivity analysis to investigate the impact of patch size and the number of filters in Figure~\ref{fig:param_sensitivity}. The analysis was performed on two datasets, namely the Cora and Texas datasets. We varied the patch size from \{8, 16, 32, 64\} and the number of filters from \{10, 50, 100, 200\}. We measure the accuracy of the model for each combination of parameters. Figure \ref{fig:param_sensitivity} shows the results of the parameter sensitivity analysis. The plots reveal that for both datasets, increasing the patch size and the number of filters generally leads to higher accuracy, although the improvements tend to plateau beyond certain values. This indicates that our model can effectively handle different patch sizes and filter configurations while highlighting the importance of selecting appropriate parameters to achieve optimal performance.
\section{Related Work}
\paragraph{Graph Neural Networks.}
Traditional Graph Neural Networks have achieved remarkable results in learning representations from graph-structured data, particularly under the assumption of homophily, where connected nodes are more likely to share the same label. These models, including Graph Convolutional Networks~\cite{kipf2017semi}, Graph Attention Networks (GATs)~\cite{velickovic2018graph}, and GraphSAGE~\cite{hamilton2017inductive}, rely heavily on localized filters or attention mechanisms for information aggregation. More recently, transformer-based models~\cite{Lin_2022_CVPR} and MLP-Mixer-inspired architecture~\cite{cong2023we, he2022generalization} have emerged, extending the capability of GNNs to capture  global positional information. Nevertheless, these models are still inherently tied to the concept of homophily and often struggle when faced with heterophilic scenarios.

\paragraph{Learning on Heterophily Graphs.}
The challenge of heterophily, where connected nodes are less likely to share the same label, necessitated the development of specialized GNNs. These heterophily GNNs, such as H2GCN~\cite{zhu2020beyond}, FAGCN~\cite{fagcn2021}, and BM-GCN~\cite{he2022block}, employ a variety of strategies, from using both low-pass and high-pass filters to feature augmentation and block-matching techniques. While they exhibit improved performance on heterophilic graphs, they often involve intricate design choices and lack robustness to different heterophily levels. In contrast, our proposed $\name$ model strikes a balance between the simplicity of MLP-Mixer architecture and the versatility of adaptive polynomial filters, efficiently utilizing both low-pass and high-pass information. More importantly, it provides a theoretical link between adaptive polynomial filters and heterophily degree, thereby offering a more generalized and adaptive solution for graph heterophily.
\section{Conclusion}

In conclusion, our study conducts a thorough theoretical analysis of the impact of graph heterophily on GNN spectral filters and highlights the necessity of designing adaptive polynomial filters to ensure generalization performance across varying degrees of heterophily. The proposed $\name$ model, inspired by the MLP-Mixer architecture, effectively addresses graph heterophily and consistently outperforms traditional GNNs and state-of-the-art heterophily GNNs in node classification tasks. This work underscores the significance of developing adaptive and robust GNN model architectures to tackle graph heterophily, paving the way for future research in graph representation learning.

\bibliographystyle{unsrt} 
\bibliography{neurips}    

\newpage
\appendix

\section{Symbols and notations}\label{sec:notation}
\begin{table}[htbp]
\caption{Symbols and notations.}
\centering
\begin{tabular}{cl}
\hline 
Symbol & Description \\
\hline
\( \mathcal{G} \) & Graph \\
\( \mathcal{V} \) & Set of vertices (nodes) of the graph \( \mathcal{G} \) \\
\( \mathcal{E} \) & Set of edges of the graph \( \mathcal{G} \) \\
\( \mathbf{X} \) & Node feature matrix of $\mathcal{G}$.\\
\( \mathcal{Y} \) & Set of labels of $\mathcal{G}$.\\
\( n \) & Number of nodes in the graph \( \mathcal{G} \) \\
\( \mathbf{L}, \mathbf{A} \) & Laplacian matrix and adjacency matrix of \( \mathcal{G} \) \\
\( \tilde{\mathbf{L}}, \tilde{\mathbf{A}} \) & Normalized laplacian matrix and adjacency matrix of \( \mathcal{G} \) \\
\( \mathcal{N}(v) \) & Set of neighbors of node \( v \) \\
\( g(\mathbf{\Lambda}) \) & A spectral graph filter of eigenvalue \( \mathbf{\Lambda} \) \\
\hline
\end{tabular}
\label{tab:notation}
\end{table}
\section{Proofs of the Theorems}\label{sec:proof}
\lemmafilter*
\begin{proof}
According to the weighted AM-GM inequality, we have
\begin{equation}
\begin{aligned}
&\frac{g(\lambda_0)}{\sum_{0\leq i\leq n-1} g(\lambda_i)}|\hat{H}_0|+\cdots +\frac{g(\lambda_{n-1})}{\sum_{0\leq i\leq n-1} g(\lambda_i)}|\hat{H}_{n-1}| \\
\geq & |\hat{H}_0|^{\frac{g(\lambda_0)}{\sum_{0\leq i\leq n-1}g(\lambda_i)}} \cdot\cdots \cdot |\hat{H}_{n-1}|^{\frac{g(\lambda_{n-1})}{\sum_{0\leq i\leq n-1}g(\lambda_i)}}
\end{aligned}
\end{equation}
As we have $g(\lambda_i) \in [0, 2]$ for all $i$, then
\begin{equation}
\begin{aligned}
& \frac{\sum_{0\leq i\leq n-1} g(\lambda_i) |\hat{H}_i|}{\sum_{0\leq i\leq n-1} g(\lambda_i)} \geq |\hat{H}_0|^{\frac{1}{\sum_{0\leq i\leq n-1}g(\lambda_i)}} \cdot\cdots \cdot |\hat{H}_{n-1}|^{\frac{1}{\sum_{0\leq i\leq n-1}g(\lambda_i)}} \\
= & \prod_{0\leq i\leq n-1} |\hat{H}_i|^{\frac{1}{\sum_{0\leq k\leq n-1}g(\lambda_k)}}
\end{aligned}
\end{equation}
Since both sides of the inequality are greater than 0, we take the logarithm
\begin{equation}
\log \left(\frac{\sum_{0\leq i\leq n-1} g(\lambda_i) |\hat{H}_i|}{\sum_{0\leq i\leq n-1} g(\lambda_i)}\right) \geq \frac{1}{\sum_{0\leq k\leq n-1}g(\lambda_k)} \sum_{0\leq i\leq n-1} \log |\hat{H}_i|
\end{equation}
Hence completes the proof.
\end{proof}

\theorembound*
\begin{proof}
Let $\psi$ be the clamp function defined as
\begin{equation}
    \psi(x) \triangleq \min \{\max \{x,-1\}, 1\}=\left\{\begin{array}{ll} 1 & x>1 \\
    x & -1<x<1 \\
    -1 & x<-1 \end{array}\right. ,
\end{equation}
\begin{equation}
\begin{aligned}
    d(x, \psi(x)) &\triangleq \left(\frac{1}{1+e^{x}}-y\right)^{2}-\left(\frac{1}{1+e^{\psi(x)}}-y\right)^{2} \\
    & \in\left\{\begin{array}{ll} {[-(\frac{1}{1+e})^{2}, 0]} & x>1, y=0 \\
    {[0,(\frac{1}{1+c})^{2}]} & x>1, y=1 \\
    {[0,(\frac{1}{1+e})^{2}]} & x<-1, y=0 \\
    {[-(\frac{1}{1+e})^{2}, 0]} & x<-1, y=1 \end{array}\right. \leq \frac{1}{(1+e)^2},
\end{aligned}
\end{equation}
and for $x\in [-1, 1]$, the first-order Taylor expansion of $\frac{1}{1+e^x}$ is $\frac{1}{2}-\frac{1}{4}x$. Denote $R(x)$ as the remainder term, that is, $R(x)=\frac{1}{1+e^x}-\frac{1}{2}+\frac{1}{4}x$. since
\begin{equation}
    \left(\frac{1}{\left(1+e^{x}\right)^{2}}\right)^{\prime \prime \prime}=-\frac{e^{x}\left(-4 e^{x}+e^{2 x}+1\right)}{\left(1+e^{x}\right)^{4}} \leq\left.\left(\frac{1}{\left(1+e^{x}\right)^{2}}\right)^{\prime \prime \prime}\right|_{x=0}=\frac{1}{8},
\end{equation}
we have
\begin{equation}
    |R(x)| \leq \max \left|\left(\frac{1}{\left(1+e^{x}\right)^{2}}\right)^{\prime \prime \prime}\right| \frac{|x|^{3}}{3 !}=\frac{|x|^{3}}{48}
\end{equation}
Therefore,
\begin{equation}
\begin{aligned}
    & \left(\frac{1}{1+e^{x}}-y\right)^{2}=\left(\frac{1}{1+e^{\psi(x)}}-y\right)^{2}+d(x, \psi(x)) \\
    = & \left(\frac{1}{2}-\frac{1}{4} \psi(x)-y+R(\psi(x))\right)^{2}+\frac{1}{(1+e)^{2}} \\
    = & \left(\frac{1}{2}-\frac{1}{4} \psi(x)-y\right)^{2} + \left(R(\psi(x))\right)^2 + 2R(\psi(x))\left(\frac{1}{2}-\frac{1}{4} \psi(x)-y\right) + \frac{1}{(1+e)^{2}} \\
    = & \left(\frac{1}{2}-\frac{1}{4} \psi(x)-y\right)^{2} + \left(\frac{|\psi(x)|^{3}}{48}\right)^2 + \frac{|\psi(x)|^{3}}{24}\left|\frac{1}{2}-\frac{1}{4} \psi(x)-y\right| + \frac{1}{(1+e)^{2}} \\
    \leq & \left(\frac{1}{2}-y\right)^{2} - \frac{(1-2y)\psi(x)}{4} + \frac{\psi(x)^2}{16} + \frac{|\psi(x)|^{6}}{2304} + \frac{|\psi(x)|^{3}}{24}\left(\frac{1}{4}|\psi(x)|+\frac{1}{2}\right) + \frac{1}{(1+e)^{2}} \\
    \leq & \frac{1}{4}-\frac{(1-2 y) \psi(x)}{4}+\frac{\psi(x)^{2}}{16}+\frac{|\psi(x)|^{3}}{48}+\frac{\psi(x)^{4}}{96}+\frac{|\psi(x)|^{6}}{2304}+\frac{1}{(1+e)^{2}}
\end{aligned}
\end{equation}
According to the above conclusion,
\begin{equation}
\begin{aligned}
    & Er\left(\mathbf{x}_0, \mathbf{y}_0\right) = \sum_l\left(\frac{1}{1+e^{g(I-\tilde{\mathbf{A}})\left(\mathbf{x}_{1 l}-\mathbf{x}_{0 l}\right)}}-\mathbf{y}_{0 l}\right)^2 \\
    \leq & \frac{n}{4}-\frac{1}{4}\left(\mathbf{y}_1-\mathbf{y}_0\right)^{\top} \psi(\mathbf{z})+\frac{\|\psi(\mathbf{z})\|_2^2}{16}+\frac{\|\psi(\mathbf{z})\|_3^3}{48}+\frac{\|\psi(\mathbf{z})\|_4^4}{96}+\frac{\|\psi(\mathbf{z})\|_6^6}{2304}+\frac{n}{(1+e)^2}
\end{aligned}
\end{equation}
$\mathbf{z}=g(I-\tilde{\mathbf{A}})(\mathbf{x}_1-\mathbf{x}_0)_l$ , noting that $C \leq\|\psi(\mathbf{z})\|_6^6 \leq\|\psi(\mathbf{z})\|_4^4 \leq\|\psi(\mathbf{z})\|_3^3 \leq\|\psi(\mathbf{z})\|_2^2 \leq n$, then we have
\begin{equation}\label{equ:signalError}
\begin{aligned}
    & Er\left(\mathbf{x}_0, \mathbf{y}_0\right) \leq \frac{n}{4}-\frac{1}{4}\left(\mathbf{y}_1-\mathbf{y}_0\right)^{\top} \psi(\mathbf{z})+\frac{217}{2304}\|\psi(\mathbf{z})\|_2^2+\frac{n}{(1+e)^2} \\
    = & c_{1} n-\frac{1}{4} \sum_l \psi\left((\mathbf{y}_{1 l}-\mathbf{y}_{0 l})(g(I-\tilde{\mathbf{A}})(\mathbf{x}_1-\mathbf{x}_0)_l)\right)
\end{aligned}
\end{equation}
where $c_1$ is a constant.  For any $\eta$, we construct $\tilde{\eta}_i=\psi_{\frac{1}{g\left(1-\lambda_i\right) \delta_i}}\left(\eta_i\right)$ such that $\left|\tilde{\eta}_i g\left(1-\lambda_i\right) \delta_i\right| \leq 1$ and $\sum_{i \in \mathcal{I}_{g, \delta, \eta}} \psi\left(\eta_i g\left(1-\lambda_i\right) \delta_i\right)=\sum_{i \in \mathcal{I}_{g, \delta, \delta, \tilde{\eta}}} \tilde{\eta}_i g\left(1-\lambda_i\right) \delta_i$. We define $m_g \triangleq \min_{i \in \mathcal{I}_{g, \delta, \eta}}\tilde{\eta}_i\delta_i$. From the proof of Lemma~\ref{lemma:filter}, for any $g(\cdot)$ and $\delta$, we have
\begin{equation}\label{equ:filterBound}
\begin{aligned}
    & \sum_{i=0}^{n-1} \psi\left(\eta_i g\left(1-\lambda_i\right) \delta_i\right)=\sum_{i \in \mathcal{I}_{g, \delta, \eta}} \psi\left(\eta_i g\left(1-\lambda_i\right) \delta_i\right)=\sum_{i \in \mathcal{I}_{g, \delta, \tilde{\eta}}} \tilde{\eta}_i g\left(1-\lambda_i\right) \delta_i \\
    \geq & m_g \sum_{i \in \mathcal{I}_{g, \delta, \tilde{\eta}}} g\left(1-\lambda_i\right) = m_g \left(\sum_{i \in \mathcal{I}_{\delta, \tilde{\eta}}} g\left(1-\lambda_i\right)+\sum_{i \in \mathcal{I}_{g}} g\left(1-\lambda_i\right)-\sum_{i=0}^{n-1} g\left(1-\lambda_i\right) \right) \\
    = & m_g \sum_{i \in \mathcal{I}_{\delta, \tilde{\eta}}} g\left(1-\lambda_i\right) \geq m_g \frac{\sum_{i \in \mathcal{I}_{\delta, \tilde{\eta}}} \log |\hat{H}_i|}{\log \sum_{i \in \mathcal{I}_{\delta, \tilde{\eta}}} g\left(1-\lambda_i\right) |\hat{H}_i|-\log \sum_{i \in \mathcal{I}_{\delta, \tilde{\eta}}} g\left(1-\lambda_i\right)}
\end{aligned}
\end{equation}
According to Eq.~\ref{equ:signalError} and Eq.~\ref{equ:filterBound}, we have
\begin{equation}
\begin{aligned}
    & \overline{Er}(\mathbf{X}, \mathbf{Y})= \frac{1}{n}\|\sigma(g(I-\tilde{\mathbf{A}}) \mathbf{X})-\mathbf{Y}\|_{F}^{2}= \frac{2}{n} Er\left(\mathbf{x}_{0}, \mathbf{y}_{0}\right) \\
    \leq & c_1 - \frac{\min_{i \in \mathcal{I}_{g, \delta, \eta}}\psi_{\frac{1}{g\left(1-\lambda_i\right) \delta_i}}\left(\eta_i\right) \cdot \delta_i \sum_{i \in \mathcal{I}_{\delta, \tilde{\eta}}} \log |\hat{H}_i|}{2n\log \sum_{i \in \mathcal{I}_{\delta, \tilde{\eta}}} g\left(1-\lambda_i\right) |\hat{H}_i| - 2n\log \sum_{i \in \mathcal{I}_{\delta, \tilde{\eta}}} g\left(1-\lambda_i\right)}
\end{aligned}   
\end{equation}
\end{proof}

\section{Algorithm}

\begin{algorithm}[H]
\caption{$\name$ Framework}
\label{algo:gpatcher-framework}
\begin{algorithmic}[1]
\Require source graph $\mathcal{G}(\mathcal{V}, \mathcal{E}, \mathbf{X})$; adjacency matrix $\mathbf{A}$; patch size $p$; layer number $m$
\Ensure $\mathcal{Y}$

\LeftComment{\textbf{Preprocessing:}}

\State $\tilde{ \mathbf{A}} \gets \text{Normalization}(\mathbf{A})$
\State $\Lambda, \mathbf{U} \gets \text{EigenDecomposition}(\tilde{\textbf{A}})$

\LeftComment{\textbf{Patch Extractor $\phi$:}}

\State $\mathbf{R} \gets \mathbf{U}g(\Lambda)\mathbf{U}^T$ \Comment{Compute patch extractor score matrix}
\State $\mathbf{P} \gets \mathbf{X}[\text{top-}p_{col}(\mathbf{R}, p)]$ \Comment{Extract top-$p$ nodes columnwise from patch extractor score matrix}
\State $\mathbf{P}^0 \gets \mathbf{P}$\Comment{$\mathbf{P} \in \mathbb{R}^{n \times p \times d}$}

\LeftComment{\textbf{Patch Mixer:}}
\For{$i \gets 1$ to $m$}
\State $\mathbf{P}^{i} \gets \sigma(\text{MLP}(\text{LayerNorm}(\mathbf{P}^{i-1}), \boldsymbol{\theta}^i_{\text{feature}}))$ \Comment{Mixing along feature dimension}
\State $\mathbf{P}^{i} \gets \sigma(\text{MLP}(\text{LayerNorm}(\mathbf{P}^{i}), \boldsymbol{\theta}^i_{\text{patch}}))$ \Comment{Mixing along patch dimension}
\EndFor
\State $\mathcal{Y} \gets \text{MLP}(\text{Aggregate}(\mathbf{P}^m), \theta_{\text{predict}})$\Comment{Aggregate along patch dimension and predict node label }
\State \Return $\mathcal{Y}$
\end{algorithmic}
\end{algorithm}
In the standard $\name$~framework, we generate patches $\mathbf{P}=\{\mathbf{P}_v: v\in\mathcal{V}\}$ for all nodes using patch extractor $\phi$, mix $\mathbf{P}$ using multiple layers of patch mixer operations to obtain $\mathbf{P}^m$ and predict node labels $\mathcal{Y}$ using aggregated $\mathbf{P}^m$. We first preprocess the data by normalizing $\textbf{A}$ and then extracting $\mathbf{\Lambda}$ and $\textbf{U}$ from $\tilde{\textbf{A}}$. We then apply the patch extractor $\phi$ and selecting on $\mathbf{X}$ to our graph, yielding patches of size $[n\times p\times d]$, where $n$ is the number of nodes, $p$ is the patch size, and $d$ is the dimension of the input node features. Each patch, represented by a $[p\times d]$ matrix, encapsulates the extracted features from a node and its top $p$ neighboring nodes. The resulting patches are subsequently processed by a feature-mixing layer, which utilizes a Multilayer Perceptron (MLP) followed by Layer Normalization and an arbitrary activation function to operate on the $d$-dimension. This results in a tensor of size $[n\times p\times d]$. To further mix information across all nodes in the patch, a patch-mixing layer is applied on the $p$-dimension using another MLP. This mixing process passes $\mathbf{P}$ through $m$ layers to construct $\mathbf{P}^m$. Finally, we aggregate along the patch dimension $p$ of $\mathbf{P}^m$ using an arbitrary aggregation function which gives us the final representation of size $[n\times d]$. By further applying MLP to this final representation, we classify the label of each node which is the size of $[n\times |\mathcal{Y}|]$.
\begin{algorithm}[H]
\caption{Fast-$\name$ Framework}
\label{algo:gpatcher-framework}
\begin{algorithmic}[1]
\Require source graph $\mathcal{G}(\mathcal{V}, \mathcal{E}, \mathbf{X})$; adjacency matrix $\mathbf{A}$; patch size $p$; layer number $m$; dangling scalar $c$
\Ensure $\mathcal{Y}$

\LeftComment{\textbf{Preprocessing:}}
\State $\tilde{ \mathbf{A}} \gets \text{Normalization}(\mathbf{A})$

\LeftComment{\textbf{Patch Extractor $\phi_{\text{fast}}$:}}
\For{$v$ in $|\mathcal{V}|$}
\State $\mathbf{r}_v  \gets \text{Minimize}~\mathcal{J}(\mathbf{r}_v)$~given $\mathbf{A},c$ \Comment{Eq~\ref{eq:ppr}}
\EndFor
\State $\text{Compute} ~\mathbf{R}~\text{by stacking}~\textbf{r}_v$
\State $\mathbf{P} \gets \mathbf{X}[\text{top-}p_{col}(\mathbf{R}, p)]$ \Comment{Extract top-$p$ nodes column-wise from patch extractor score matrix}
\State $\mathbf{P}^0 \gets \mathbf{P}$\Comment{$\mathbf{P} \in \mathbb{R}^{n \times p \times d}$}

\LeftComment{\textbf{Patch Mixer:}}
\For{$i \gets 1$ to $m$}
\State $\mathbf{P}^{i} \gets \sigma(\text{MLP}(\text{LayerNorm}(\mathbf{P}^{i-1}), \boldsymbol{\theta}^{i}_{\text{feature}}))$ \Comment{Mixing along feature dimension}
\State $\mathbf{P}^{i} \gets \sigma(\text{MLP}(\text{LayerNorm}(\mathbf{P}^{i}), \boldsymbol{\theta}^{i}_{\text{patch}}))$ \Comment{Mixing along patch dimension}
\EndFor
\State $\mathcal{Y} \gets \text{MLP}(\text{Aggregate}(\mathbf{P}^m), \theta_{\text{predict}})$\Comment{Aggregate along patch dimension and predict node label }
\State \Return $\mathcal{Y}$
\end{algorithmic}
\end{algorithm}
In the Fast-$\name$ framework, we generate patches $\mathbf{P}=\{\mathbf{P}_v: v\in\mathcal{V}\}$ for all nodes using patch extractor $\phi_{\text{fast}}$, mix $\mathbf{P}$ using multiple layers of patch mixer operations to obtain $\mathbf{P}^m$, and predict node labels $\mathcal{Y}$ using aggregated $\mathbf{P}^m$. The patch extraction process involves minimizing an objective function, as described in Eq~\ref{eq:ppr}, to obtain the patch extractor scores $\mathbf{r}_v$ for each node $v$. These scores are then stacked to form the patch extractor score matrix $\mathbf{R}$. The dangling scalar $c$ controls the balance between local and global information integration during the patch extractor score computation. When $c$ is lower, the focus of the patch extractor score matrix $\mathbf{R}$ is more localized, prioritizing local structural information. On the other hand, when $c$ is higher, the patch extractor score matrix $\mathbf{R}$ emphasizes global structural information and captures more long-range dependencies in the graph. The choice of $c$ allows for a trade-off between local and global information integration, enabling the framework to adapt to different graph characteristics and learning objectives. The rest of the framework remains the same as the standard $\name$ framework.

\section{Statistics of Datasets}\label{sec:data}

\begin{table}[H]
\centering
\begin{tabular}{cc|c|c|c|c|c}
\hline
\multicolumn{2}{c|}{\textbf{Dataset}} & \textbf{$\mathcal{|V|}$} & \textbf{$\mathcal{|E|}$} & \textbf{$|\mathcal{Y}|$} & $|\mathbf{X}|$ & $\frac{\sum{H_i}}{n}$ \\ 
\hline
\multicolumn{1}{c|}{\multirow{3}{*}{\textbf{Homophily}}} & \textbf{Cora} & 2708 & 5278 & 7 & 1433 & 0.19 \\ 
\multicolumn{1}{c|}{} & \textbf{Citeseer} & 3327 & 4676 & 6 & 3703 & 0.26 \\ 
\multicolumn{1}{c|}{} & \textbf{Pubmed} & 19717 & 44327 & 3 & 500 & 0.2 \\ 
\hline
\multicolumn{1}{c|}{\multirow{6}{*}{\textbf{Heterophily}}} & \textbf{Texas} & 183 & 295 & 5 & 1703 & 0.89 \\ 
\multicolumn{1}{c|}{} & \textbf{Squirrel} & 5201 & 198493 & 5 & 2089 & 0.78 \\ 
\multicolumn{1}{c|}{} & \textbf{Chameleon} & 2277 & 31421 & 5 & 2325 & 0.77 \\ 
\multicolumn{1}{c|}{} & \textbf{Cornell} & 183 & 295 & 5 & 1703 & 0.89 \\ 
\multicolumn{1}{c|}{} & \textbf{Wisconsin} & 251 & 499 & 5 & 1703 & 0.84 \\ 
\multicolumn{1}{c|}{} & \textbf{Actor} & 7600 & 33544 & 5 & 931 & 0.76 \\ 
\hline
\end{tabular}
\caption{Dataset Statistics of homophily and heterophily graph datasets. The table is divided into two sections. The top section represents homophily graph datasets, where the graphs exhibit similar attributes and connectivity patterns within nodes of the same class. The bottom section represents heterophily graph datasets, where the graphs exhibit diverse attributes and connectivity patterns across different classes of nodes. The columns show the number of nodes ($\mathcal{|V|}$), the number of edges ($\mathcal{|E|}$), the number of unique classes ($|\mathcal{Y}|$), the feature dimension ($|\mathbf{X}|$), and the average heterophily score ($\frac{\sum{H_i}}{n}$) indicating the level of heterophily within each dataset.}
\label{tab:dataset_char}
\end{table}

\section{Details of Parameter Settings and Tuning}
 We perform a grid search to find the optimal hyperparameters for our model, including the learning rate, patch size, and filter number. The search space is defined as follows: the learning rate $\in \{0.01, 0.008, 0.005, 0.003, 0.001\}$, the patch size $\in \{8, 16,32,64,96\}$, and the filter number $\in \{10,50,100,150,200\}$. For Fast-$\name$, we set $c$ as 0.5 to equally weigh local and global structural information. The optimal hyperparameters are determined based on the performance of the validation set.

\end{document}